\documentclass[sigconf]{aamas}  % do not change this line!

\AtBeginDocument{%
  \providecommand\BibTeX{{%
    \normalfont B\kern-0.5em{\scshape i\kern-0.25em b}\kern-0.8em\TeX}}}
    
\usepackage{booktabs}    
\usepackage{flushend} % do not change this line!

\setcopyright{ifaamas}  % do not change this line!
\copyrightyear{2020} % do not change this line!
\acmYear{2020} % do not change this line!
\acmDOI{} % do not change this line!
\acmPrice{} % do not change this line!
\acmISBN{} % do not change this line!
\acmConference[AAMAS'20]{Proc.\@ of the 19th International Conference on Autonomous Agents and Multiagent Systems (AAMAS 2020)}{May 9--13, 2020}{Auckland, New Zealand}{B.~An, N.~Yorke-Smith, A.~El~Fallah~Seghrouchni, G.~Sukthankar (eds.)}  % do not change this line!

\usepackage[linesnumbered, ruled,vlined]{algorithm2e} % For algorithms
\usepackage{graphicx}
\usepackage{booktabs}
\usepackage{subfig}
\usepackage{amsmath}
\usepackage{amsfonts}
\usepackage{wrapfig}
\usepackage{color}
\usepackage{multirow}

\begin{document}

\title{Multi-Path Policy Optimization}  % put your title here!

\author{Ling Pan}
\affiliation{%
  \institution{IIIS, Tsinghua University}
    \city{Beijing}
  	\country{China}
}
\email{pl17@mails.tsinghua.edu.cn}
\author{Qingpeng Cai}
\affiliation{%
  \institution{Alibaba Group}
  \city{Beijing}
  \country{China}
}
\email{qingpeng.cqp@alibaba-inc.com}
\author{Longbo Huang}
\affiliation{%
  \institution{IIIS, Tsinghua University}
    \city{Beijing}
    \country{China}
}
\email{longbohuang@tsinghua.edu.cn}

\begin{abstract}
Recent years have witnessed a tremendous improvement of deep reinforcement learning. However, a challenging problem is that an agent may suffer from inefficient exploration, particularly for on-policy methods. Previous exploration methods either rely on complex structure to estimate the novelty of states, or incur sensitive hyper-parameters causing instability. We propose an efficient exploration method, Multi-Path Policy Optimization (MPPO), which does not incur high computation cost and ensures stability. MPPO maintains an efficient mechanism that effectively utilizes a population of diverse policies to enable better exploration, especially in sparse environments. We also give a theoretical guarantee of the stable performance. We build our scheme upon two widely-adopted on-policy methods, the Trust-Region Policy Optimization algorithm and Proximal Policy Optimization algorithm. We conduct extensive experiments on several MuJoCo tasks and their sparsified variants to fairly evaluate the proposed method. Results show that MPPO significantly outperforms state-of-the-art exploration methods in terms of both sample efficiency and final performance.
\end{abstract}

 % put your comma-separated keywords here!

\maketitle

\section{Introduction}
In reinforcement learning, an agent seeks to find an optimal policy that maximizes long-term rewards by interacting with an unknown environment.
Policy-based methods, e.g., \cite{lillicrap2015continuous,fujimoto2018addressing,mnih2016asynchronous}, optimize parameterized policies by gradient ascent on the performance objective.
Directly optimizing the policy by vanilla policy gradient methods may incur large policy changes, which can result in performance collapse due to unlimited updates.
To resolve this issue, Trust Region Policy Optimization (TRPO) \cite{schulman2015trust} and Proximal Policy Optimization (PPO) \cite{schulman2017proximal} optimize a surrogate function in a conservative way, both being on-policy methods that perform policy updates based on samples collected by the current policy.
These on-policy methods have the desired feature that they generally achieve stable performance.
In contrast,  off-policy learning, where policies are updated according to samples drawn from a different policy, e.g., using an experience replay buffer \cite{haarnoja2018soft}, can often suffer from practical convergence and stability issue \cite{gu2016q,gu2017interpolated}.
However, as on-policy methods learn from what they collect, they can particularly suffer from insufficient exploration ability, especially in sparse environments \cite{colas2018gep}.
Thus, although TRPO and PPO start from a stochastic policy, the randomness in the policy decreases quickly during training. 
As a result, they can converge too prematurely to bad local optima in high-dimensional or sparse reward tasks.

Indeed, how to achieve efficient exploration is challenging in deep reinforcement learning.
There has been recent progress in improving exploration ranging from count-based exploration \cite{ostrovski2017count,tang2017exploration,fu2017ex2}, intrinsic motivation \cite{houthooft2016vime,bellemare2016unifying,pathak2017curiosity}, to noisy networks \cite{fortunato2018noisy,plappert2018parameter}.
However, these methods either introduce sensitive parameters that require careful tuning on tasks \cite{khadka2018evolution}, or require learning additional complex structures to estimate the novelty of states.
Another line of research \cite{hong2018diversity,masood2019diversity} proposes to encourage exploration by augmenting the objective function with a diversity term that measures the distance of current and prior policies. 
Yet, the distance between the current policy and past polices can be small for trust-region methods where the policy update is controlled, and limits the applicability of the diversity-driven approaches.

Vanilla evolutionary algorithms \cite{spears1993overview}, e.g., population-based methods \cite{zames1981genetic}, on the other hand, exhibit great exploration ability and stability with the maintenance  of a population of agents, which are able to collect diverse samples \cite{khadka2018evolution,chang2018genetic}.
However, evolutionary algorithms are sample-inefficient compared with deep reinforcement learning approaches \cite{colas2018gep}, as they learn from the whole episode instead of single steps \cite{sigaud2019policy}, and does not exploit the powerful gradient information \cite{khadka2018evolution}.

Recent research has shown great potential in combining deep reinforcement learning with population-based methods to reap the best from both families of algorithms \cite{jaderberg2017population,khadka2018evolution,khadka2019collaborative,pourchot2018cem,leibo2019malthusian}, which mainly focuses on off-policy learning.
These approaches generally maintains a population of agents, and each of them interacts with the environment to collect experiences.
The key to success is that these diverse experiences are stored in a shared experience replay buffer, which can be utilized by off-policy algorithms.
However, it is sample-inefficient to apply population-based methods to on-policy learning.
This is due to the nature of on-policy methods, where the policy can only be updated by samples collected by itself.
Therefore, rolling out all policies in the population at each iteration as in \cite{khadka2018evolution,khadka2019collaborative,pourchot2018cem} can be inefficient as each policy cannot exploit other policies' experiences.
In addition, each interaction with the environment can be expensive \cite{buckman2018sample}. 

To enable an effective and efficient combination of on-policy reinforcement learning algorithms and population-based methods, in this paper, we propose a novel method, Multi-Path Policy Optimization (MPPO), which improves exploration for on-policy algorithms using multiple paths.
Here, a path refers to a sequence of policies generated during the course of policy optimization starting from a single policy.
Figure \ref{mppo_framwork} demonstrates the high-level schematic of MPPO, which has four main components, i.e., {\em pick and rollout}, {\em value function approximation}, {\em policy optimization}, and {\em policy buffer update}.
Specifically, MPPO starts with $K$ different policies randomly initialized in the policy buffer, and a shared value network.
At each iteration, a candidate policy is picked from the policy buffer according to a picking rule, defined as a weighted combination of performance and entropy, introduced to enable a trade-off between the exploration and exploitation.
Then, the picked policy interacts with the environment by rollouts to collect samples.
The shared value network is updated based on these samples to approximate the value function.
The picked policy is updated by policy optimization according to the samples and the shared value network.
Finally, the improved picked policy updates the policy buffer by replacing itself, in order to retain the diversity of the policy buffer.

\begin{figure}
\centering
\includegraphics[scale=0.295]{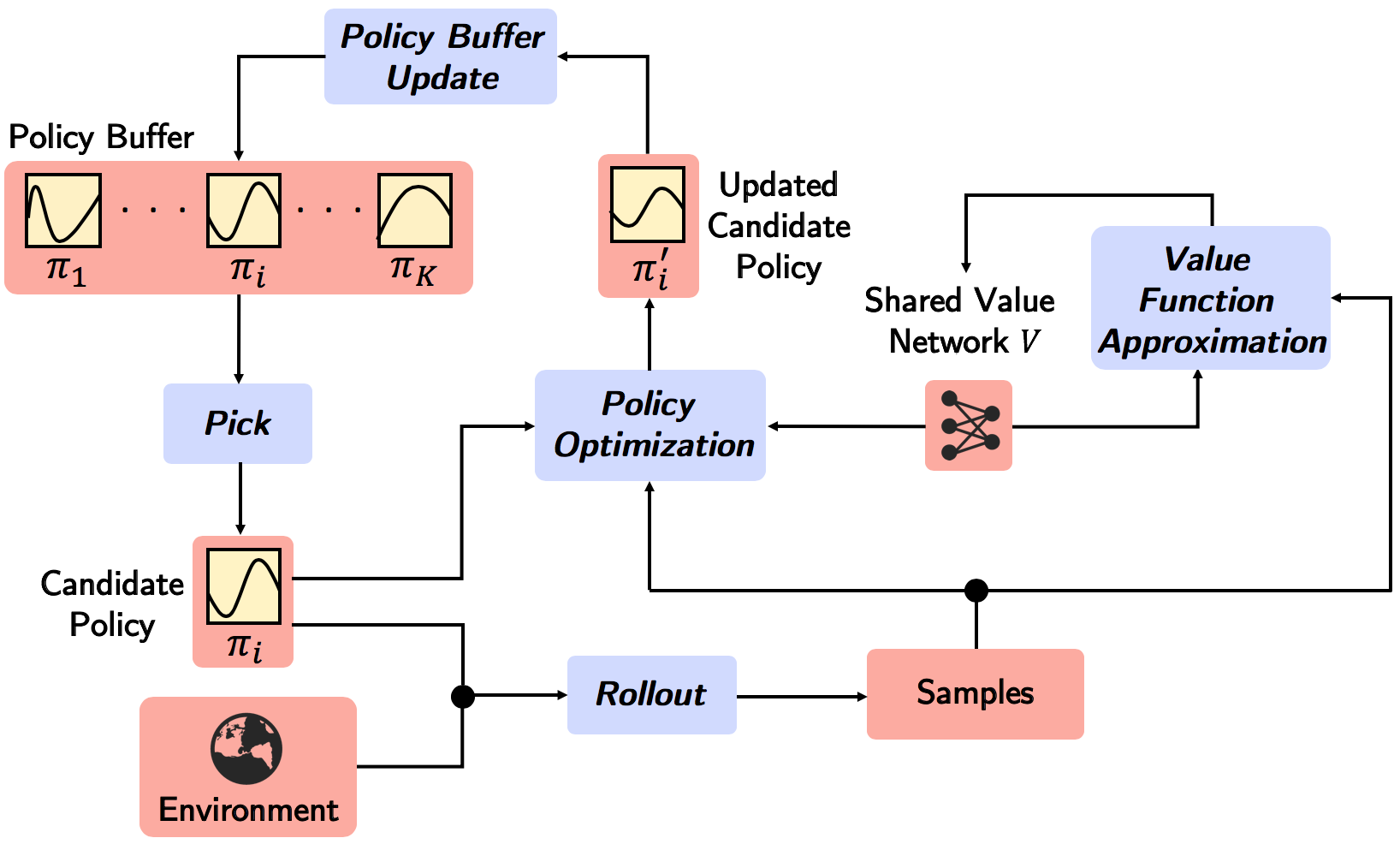}
\caption{High-level schematic of MPPO.}
\label{mppo_framwork}
\end{figure}

With this scheme, MPPO maintains $K$ policy paths, which increases the exploration ability during training.
Different policy paths provide diverse experiences for the shared value network to enable a better estimation \cite{nachum2017trust}, which yields a better signal for telling how well each state is.
With a better estimated value function, the picked updated by policy optimization are more able to collect trajectories with higher rewards.
Therefore, MPPO can provide better guidance for the picked policy, which MPPO aims to optimize, to explore states and actions that were not known to have high rewards previously.
Moreover, since only one candidate policy is picked and optimized at each iteration, our method does not incur much computational cost compared with the base policy optimization method.

A critical component of MPPO is the picking rule, which favors to select the policy that is most desirable to rollout and to optimize at each iteration, i.e., the one with good performance while being explorative simultaneously.
We prove that when MPPO switches to an explorative policy, the performance variation of picked polices can be bounded and controlled. 
This is a useful feature that ensures smooth policy transition.
We also empirically validate that the potential variation is small, and the picked policy converges to one single policy, which ensures the stability.

We apply MPPO to two widely adopted on-policy algorithms, TRPO \cite{schulman2015trust} and PPO \cite{schulman2017proximal}, and conduct extensive experiments on several continuous control tasks based MuJoCo \cite{todorov2012mujoco}.
Experimental results demonstrate that our proposed algorithms, MP-TRPO and MP-PPO, provide significant improvements over state-of-the-art exploration methods, in terms of sample efficiency and final performance without incurring high computational cost.
We also analyze the effect of each component in our methodology and investigate the critical advantages of the proposed picking rule and policy buffer update strategy.

The main contributions can be summarized as follows:
\begin{itemize}
\item We propose a novel methodology utilizing a population of policies to tackle the exploration bottleneck of on-policy reinforcement learning algorithms, which is efficient and effective without introducing much computation costs.
\item We give a theoretical guarantee of stable performance of Multi-Path TRPO (MP-TRPO).
\item MPPO can be readily applied given any baseline on-policy algorithm. We validate MPPO to two popular on-policy algorithms, TRPO and PPO, and conduct extensive evaluation on a wide range of MuJoCo tasks. Results show that MPPO outperforms state-of-the-art exploration methods.
\end{itemize}

\section{Preliminaries}
A Markov decision process (MDP) is defined by $(\mathcal{S}, \mathcal{A}, p, r, \gamma)$,
where $\mathcal{S}$, $\mathcal{A}$ denote the set of states and actions, $p(s'|s, a)$ the transition probability from state $s$ to state $s'$ under action $a$, $r(s, a)$ the corresponding immediate reward, and $\gamma \in [0, 1)$ the discount factor.
The agent interacts with the environment by its parameterized policy $\pi_{\theta}$, with the goal to learn the optimal policy that maximizes the expected discounted return $J(\pi_\theta) = \mathbb{E}[\sum_{t=0}^{\infty} \gamma^t r_t|\pi_{\theta}]$.

Trust Region Policy Optimization (TRPO) \cite{schulman2015trust} learns the policy parameter by optimizing a surrogate function in a conservative way.
Specifically, it limits the stepsize towards updating the policy using a trust-region constraint, i.e.,
\begin{align}
\max_{\theta}  \quad & \mathcal{L}_{\pi_{\theta_{\rm{old}}}}(\pi_\theta) = {\mathbb{E}}_t \left[ \frac{\pi_{\theta}(a_t|s_t)}{\pi_{\theta_{\rm{old}}}(a_t|s
_t)} A_t^{\pi_{\theta_{\rm{old}}}}(s_t,a_t) \right] \label{trpo_obj}\\
{\rm s.t.} \quad & {\mathbb{E}}_t \left[ D_{KL} \left( \pi_{\theta}(\cdot|s_t) || \pi_{\theta_{\rm{old}}}(\cdot | s_t) \right) \right] \leq \delta,
\label{eq:trpo_program}
\end{align}
where ${\mathbb{E}}_t[...]$ is the empirical average over a finite batch of samples,
$A_t^{\pi_{\theta_{\rm{old}}}}(s_t,a_t)=Q_t^{\pi_{\theta_{\rm{old}}}}(s_t,a_t)-V_t^{\pi_{\theta_{\rm{old}}}}(s_t)$, and
$Q_t^{\pi_{\theta_{\rm{old}}}}(s_t,a_t) = r(s_t, a_t) + \gamma \mathbb{E}_{s_{t+1}}\left[V_t^{\pi_{\theta_{\rm{old}}}}(s_{t+1}) \right]$.
One desired feature of TRPO is that it guarantees a monotonic policy improvement, i.e., the policy update step leads to a better-performing policy during training.
However, it is not computationally efficient as it involves solving a second-order optimization problem using conjugate gradient.

Proximal Policy Optimization (PPO) \cite{schulman2017proximal} is a simpler method only involving first-order optimization using stochastic gradient descent.
PPO maximizes a KL-penalized or clipped version of the objective function to ensure stable policy updates, where the clipped version is more common and is reported to perform better than the KL-penalized version.
Specifically, the objective for the clipped version is to maximize
\begin{equation}
\begin{split}
\mathcal{L}_{\pi_{\theta_{\rm{old}}}}(\pi_{\theta}) = {\mathbb{E}}_t & \left[  \min ( r_t(\pi_{\theta_{\rm{old}}},\pi_{\theta}) {A}_t, {\rm clip} ( r_t(\pi_{\theta_{\rm{old}}},\pi_{\theta}), \right. \\
& \left. 1 - \epsilon, 1 + \epsilon ) {A}_t  ) \right], 
\end{split}
\label{eq:ppo}
\end{equation}
where $r_t(\pi_{\theta_{\rm{old}}},\pi_{\theta})=\frac{\pi_{\theta}(a_t | s_t)}{\pi_{\theta_{\rm{old}}}(a_t | s_t)}$ denotes the probability ratio, and $\epsilon$ is the parameter for clipping. 

\section{Multi-Path Policy Optimization}

\subsection{A Motivating Example}
Figure \ref{fig:maze}(a) shows a challenging environment Maze of size $21 \times 21$ with sparse rewards, where black lines in the middle represent walls. 
The agent always starts at $S$ located on the lower left corner of the maze with the goal of reaching the destination $G$ in the lower right corner, where the maximum length of an episode is 1000.
A reward of $+1$ is given only when the agent reaches $G$, and $0$ otherwise.
The experimental setting in Maze is the same as in Section 4.1.

\begin{figure}[!h] 
\centering
\subfloat[Maze.]{\includegraphics[width=0.255\linewidth]{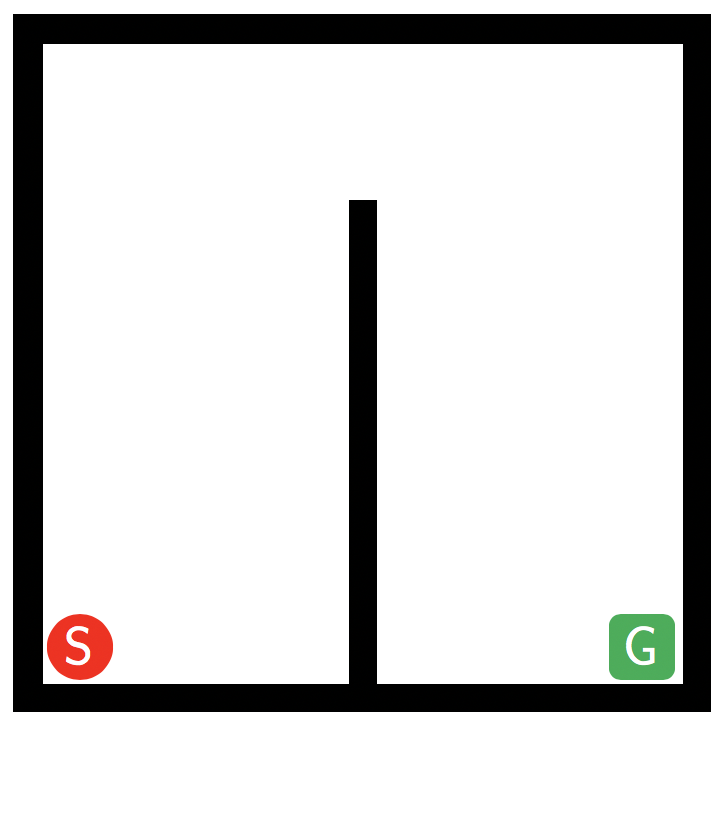}}
\subfloat[Return.]{\includegraphics[width=0.37\linewidth]{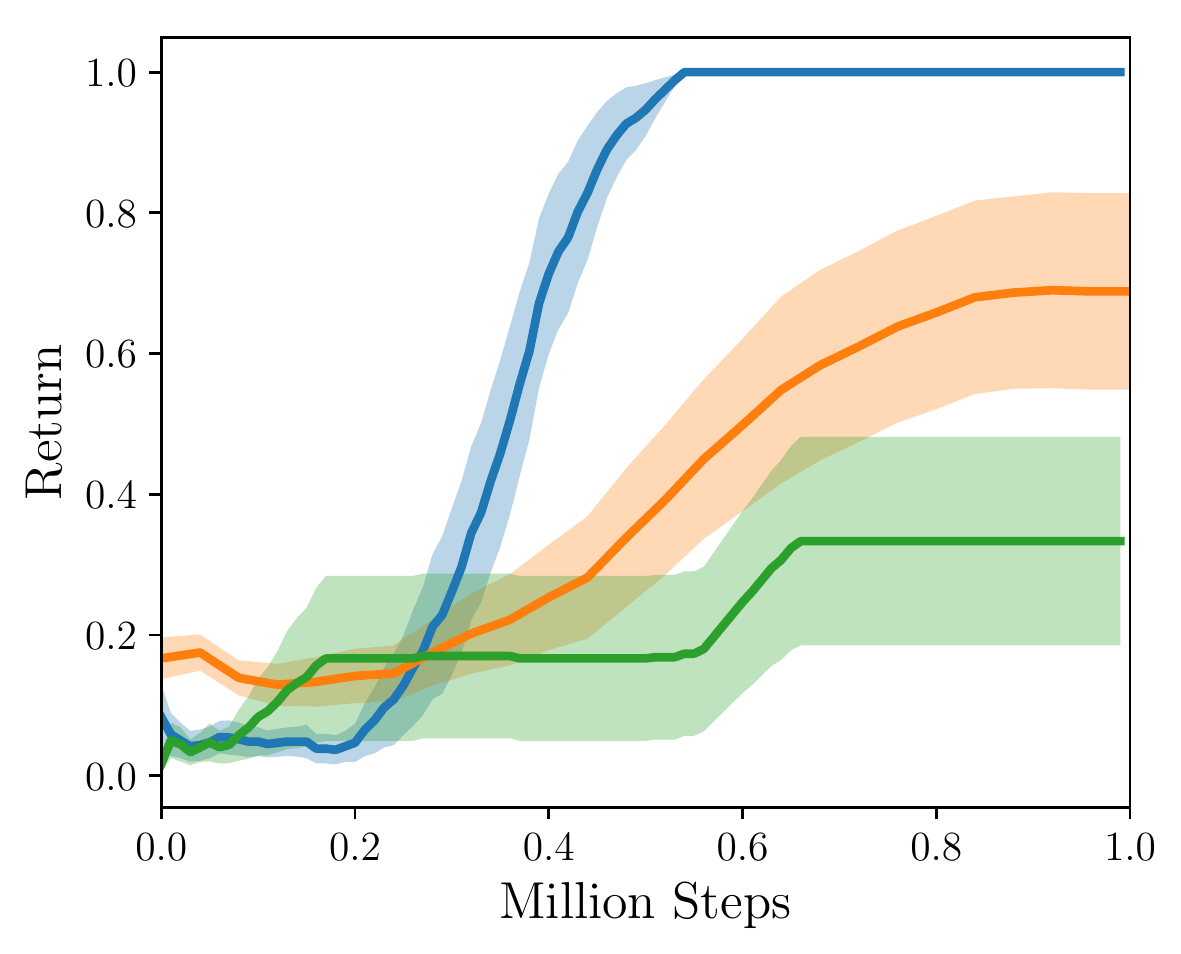}}
\subfloat[Entropy.]{\includegraphics[width=0.37\linewidth]{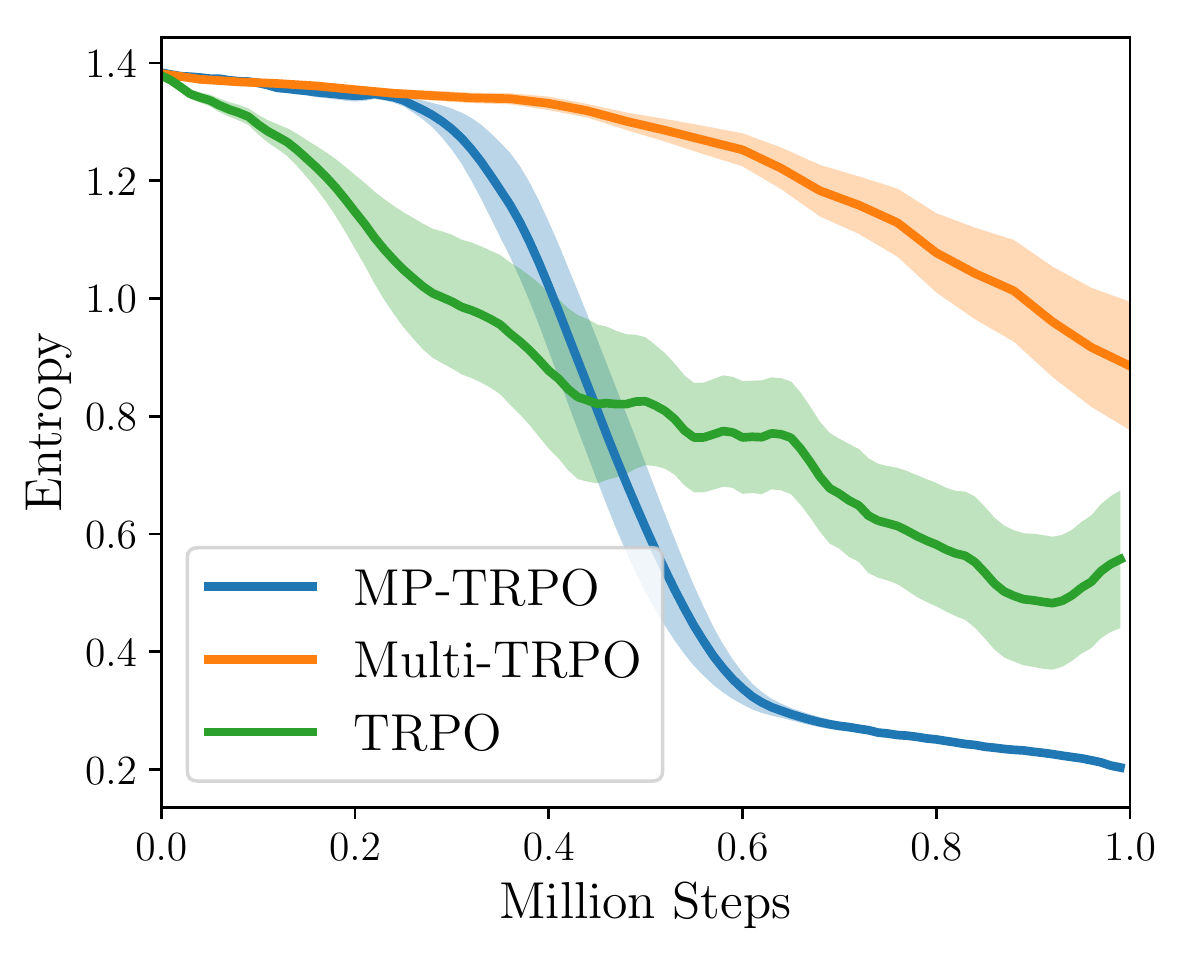}}
\caption{Performance and entropy comparisons on Maze.}
\label{fig:maze}
\end{figure}

We compare three schemes in this environment, i.e., TRPO, MP-TRPO (MPPO applied to TRPO) and Multi-TRPO (training a population of policies and picking the best one).
For fair comparison, all methods use the same amount of samples during training.
As shown in Figure \ref{fig:maze}(b), the agent can fail to reach the goal state under TRPO.
Figure \ref{fig:state_vis}(a) shows the resulting state visitation density under TRPO after training for 1 million steps.
Specifically, the brightness of a region in Figure \ref{fig:state_vis} indicates the number of times the agent visits that region, i.e., the brighter the region is, the more times the agent visits that region. 
It can be seen that the agent can only explore a very limited area in the maze and mainly stays in the left side.
It is also worth noting that simply training the ensemble of policies and choosing the best, i.e., Multi-TRPO, also fails to consistently find the destination.
Although it is able to search a larger region, it still mostly re-explores the left part as shown in Figure \ref{fig:state_vis}(b).
In contrast, MP-TRPO can always successfully reach the destination after 0.6 million steps while others fail.
As illustrated in Figure \ref{fig:state_vis}(c), it is capable to bypass the wall and explore both sides of the maze.

\begin{figure}[!h] 
\centering
\subfloat[TRPO.]{\includegraphics[width=0.3\linewidth]{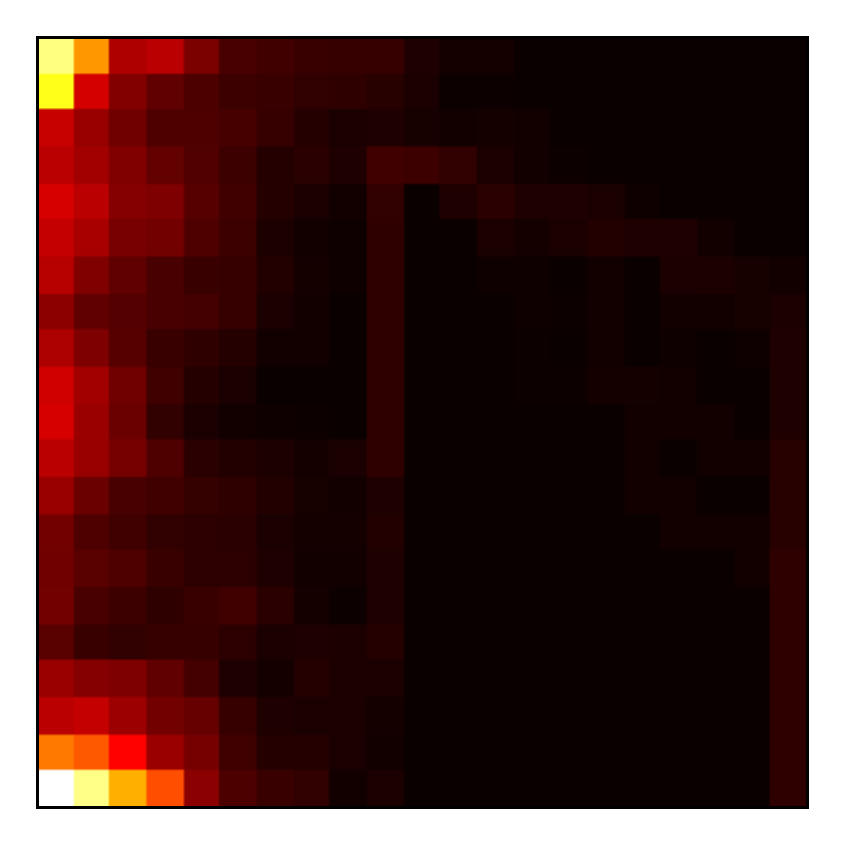}}
\subfloat[Multi-TRPO.]{\includegraphics[width=0.3\linewidth]{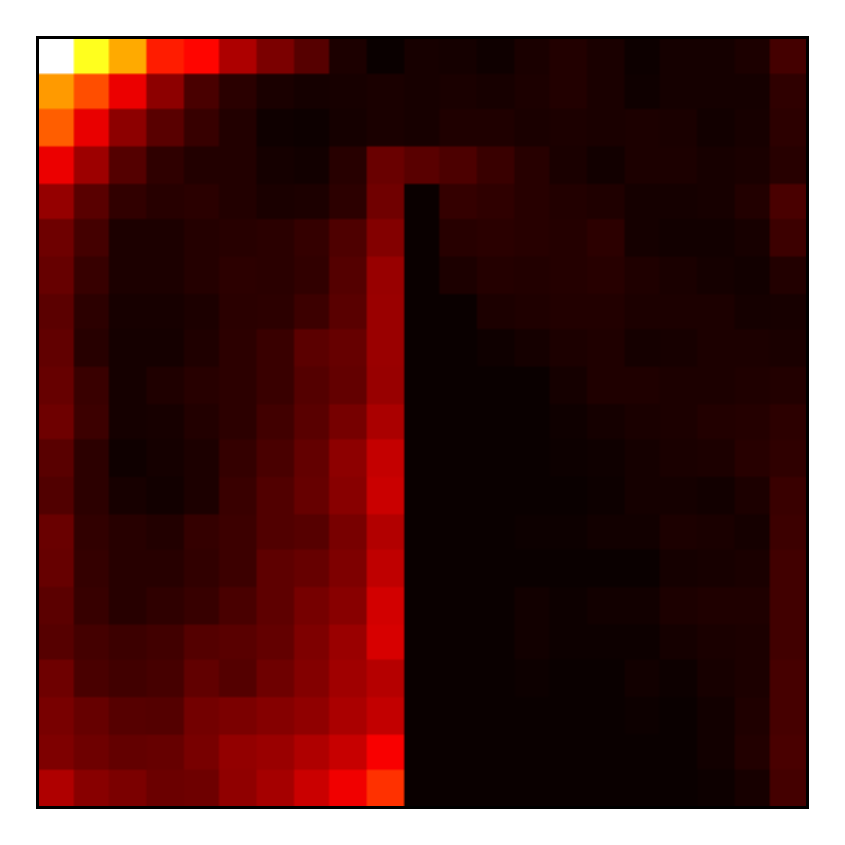}}
\subfloat[MP-TRPO.]{\includegraphics[width=0.3\linewidth]{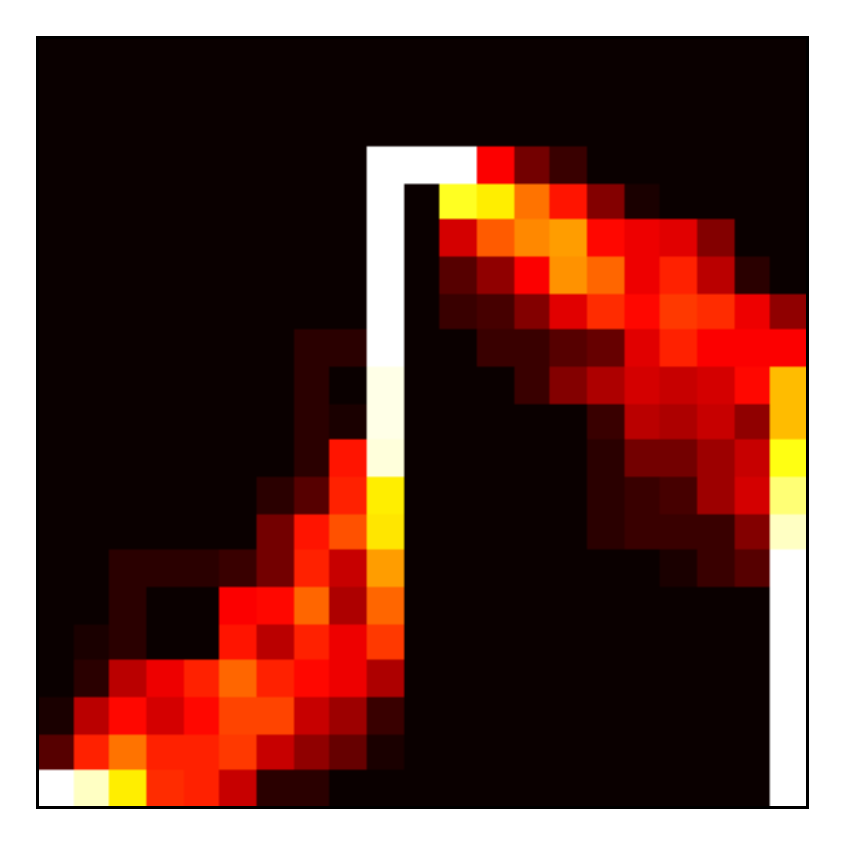}}
\caption{State visitation comparisons.}
\label{fig:state_vis}
\end{figure}

This is because TRPO suffers from insufficient exploration, and the entropy of the policy trained with TRPO decreases quickly as the policy is being optimized, as shown in Figure \ref{fig:maze}(c). 
Multi-TRPO maintains greater exploration ability with the population of policies.
However, recall that all three schemes consume the same amount of samples for training.
As it rolls out all policies at each iteration, the performance improvement of any single policy in the population is limited compared with MP-TRPO.
Indeed, for Multi-TRPO, the acquirement of diverse samples from the policy buffer comes at the expense of insufficient training of each policy under limited number of samples.
This is because on-policy algorithms cannot utilize experiences from other policies, and can only update the policy based on samples collected by itself. 
On the other hand, during the training process, MP-TRPO optimizes the policy while simultaneously maintaining enough exploration ability.

We next systematically describe our proposed method and the motivation behind it, to illustrate why MPPO helps to improve exploration.

\begin{algorithm*}[!h]
\KwIn{Initial policy buffer $\pi_0=( \pi_{10}, ..., \pi_{K0} )$ for $K$ initial policies with parameters $\theta_0=( \theta_{10}, ..., \theta_{K0} )$, initial value function $V_{\phi_0}$}
\For {$t = 0, 1, ...$}{
  Normalize $J(\pi_t)$ and $\mathcal{H}(\pi_t)$ according to Eq. (\ref{eq:normalized_J_H}) by $\forall k, \ \hat{J}_k(\pi_t) = \frac{J_k(\pi_t) - \min_{m} J_m(\pi_t)}{\max_m J_m(\pi_t) - \min_m J_m(\pi_t)}, \hat{\mathcal{H}}_k(\pi_t) = \frac{\mathcal{H}_k(\pi_t) - \min_m \mathcal{H}_m(\pi_t)}{\max_m \mathcal{H}_m(\pi_t) - \min_m \mathcal{H}_m(\pi_t)}.
$ \\
  Compute scores $f_k(\pi_t)$ according to Eq. (\ref{eq:criterion}) by $\forall k, \ f_k(\pi_t) = (1-\alpha) \hat{J}_k(\pi_t) + \alpha \hat{\mathcal{H}}_k(\pi_t)$\\
  Select the candidate policy $\pi_{it}$ where $ i= \arg \max_k f_k(\pi_t)$ \\
  Collect set of trajectories $\mathcal{D}$ by rolling out policy $\pi_{it}$ in the environment \\
  Evaluate the performance of the candidate policy $J_i(\pi_{t})$ based on the collected trajectories $\mathcal{D}$ \\ 
  Update the value function $V_{\phi_t}$ by regression on mean-squared error \\
  Compute advantage estimates ${A}_{\pi_{it}}$ using generalized advantage approach \cite{schulman2015high} based on $V_{\phi_t}$ \\
  Update the candidate policy parameter from $\pi_{\theta_{it}}$ to $\pi_{\theta_{i(t+1)}}$ by the base policy optimization method, e.g., TRPO or PPO \\
  Compute the performance gain $\mathcal{G}_{\pi_{\theta_{it}}}(\pi_{\theta_{i(t+1)}})$ according to Eq. (\ref{trpo_obj}) by $\hat{\mathbb{E}}_t \left[ \frac{\pi_{\theta_{i(t+1)}}(a_t|s_t)}{\pi_{\theta_{it}}(a_t|s_t)} A_t^{\pi_{\theta_{it}}}(s_t,a_t) \right]$\\ 
  Update the policy buffer by $\pi_{t+1} = ( \pi_{1t}, ..., \pi_{it}^{'}, ..., \pi_{Kt} )$ and the entropy buffer accordingly \\
  Update the performance buffer by $J(\pi_{t+1}) = \left( J_1(\pi_t), ..., J_i(\pi_t)+\mathcal{G}_{\pi_{\theta_{it}}}(\pi_{\theta_{i(t+1)}}), ..., J_K(\pi_t) \right)$ \\
\caption{Multi-Path Policy Optimization algorithm.}
\label{alg:mp_trpo}}
\end{algorithm*}

\subsection{Method}
The main idea of MPPO is summarized as follows.
The policy buffer is initialized with $K$ random policies, and a shared value network $V$ is also randomly initialized.
At each iteration $t$, a candidate policy $\pi_{it}$ is picked from the policy buffer $\pi_t=(\pi_{1t},...,\pi_{Kt})$, which is used as the rollout policy to interact with the environment to generate a set of samples.
The collected samples contribute to updating the shared value network for value function approximation.
The candidate policy is optimized according to the collected samples based on the shared value network.
Finally, the improved candidate policy $\pi_{it}^{'}$ updates the policy by replacing itself.

Specifically, the key components of the Multi-Path Policy Optimization method are as follows:

\subsubsection {{Pick and rollout.}} 
From previous analysis of Multi-TRPO, although a population of policies can bring diverse samples, policies in the population cannot directly exploit others' experiences in on-policy learning.
Therefore, it is unnecessary to rollout all policies in the population to interact with the environment for samples collection at each iteration.
In order to guarantee sample efficiency, we propose to pick a candidate policy from the current policy buffer at each iteration.

One common way is to pick a candidate policy randomly as in \cite{osband2016deep}.
However, this picking rule fails to fully utilize the policy buffer. 
This is because it can hardly provide guidance for the agent to pick the policy that is most desirable to rollout and to optimize, as each interaction with the environment can be expensive \cite{buckman2018sample}. 

The picking rule for MPPO is to choose the policy $\pi_{it}$ with highest score $f_i$, which takes into account both performance and entropy as defined in Eq. (\ref{eq:criterion}), i.e.,
\begin{equation}
\forall k, \ f_k(\pi_t) = (1-\alpha) \hat{J}_k(\pi_t) + \alpha \hat{\mathcal{H}}_k(\pi_t),
\label{eq:criterion}
\end{equation}
where $\hat{J}$ and $\hat{\mathcal{H}}$ denote the normalized performance and entropy according to min-max normalization as in Eq. (\ref{eq:normalized_J_H}).
\begin{equation}
\begin{split}
\forall k, \ & \hat{J}_k(\pi_t) = \frac{J_k(\pi_t) - \min_{m} J_m(\pi_t)}{\max_m J_m(\pi_t) - \min_m J_m(\pi_t)}, \\
& \hat{\mathcal{H}}_k(\pi_t) = \frac{\mathcal{H}_k(\pi_t) - \min_m \mathcal{H}_m(\pi_t)}{\max_m \mathcal{H}_m(\pi_t) - \min_m \mathcal{H}_m(\pi_t)}.
\end{split}
\label{eq:normalized_J_H}
\end{equation}
In Eq. (\ref{eq:normalized_J_H}), $J_k(\pi_t)$, $\mathcal{H}_k(\pi_t)$ denote the performance and entropy of policy $\pi_{kt}$ respectively, where we use Shannon entropy defined by $\mathcal{H}_k(\pi_t) = \mathbb{E}_{\pi_{kt}} \left[ - \log \pi_{kt}(a|s) \right]$, and other forms of entropy can also be used, e.g., Tsallis entropy \cite{chow2018path}.

The picking rule favors to pick the policy that is most desirable to rollout and to optimize, i.e., the one with good performance while being explorative simultaneously, which is a critical component of MPPO.
In Eq. (\ref{eq:criterion}), $\alpha$ provides the trade-off between exploration and exploitation. 
Note that a criterion focusing only on the performance cannot make good use of the policy buffer, as it tends to pick the policy updated in last iteration.
Therefore, it leads to a similar optimization process as that of single-path, which also suffers from insufficient exploration.
Considering the entropy term encourages exploring new behaviors.
However, if one always pick the policy with the maximum entropy, it fails to exploit learned good behaviors.
Our weighted rule is designed to strike for a good tradeoff between exploration and exploitation.

\subsubsection {{Value function approximation.}}
Samples collected by the candidate policy contribute to updating the shared value network to approximate the value function by minimizing the mean-squared error:
$\frac{1}{N} \sum_{n=1}^N {(r_n+\gamma V_{{\phi}_{t}}(s_{n+1}) - V_{{\phi}_{t}}(s_n))}^{2}.$
During the course of training with MPPO, the shared value network exploits diverse samples collected by policies that are most desirable to be picked from the diverse policy buffer at each iteration. 
In this way, it can better estimate the value function compared with that of single-path.
Therefore, it provides more information for the advantage function to distinguish good or bad actions, which is critical and helpful for policy optimization.

\subsubsection {{Policy optimization.}}
At each iteration $t$, only the candidate policy $\pi_{it}$ is optimized using a base policy optimization method, according to samples collected by itself and the shared value network.
\footnote{Note that MPPO aims to optimize the picked policy instead of all policies in the population.}
The objective of policy optimization is to maximize the expected advantages over the policy distribution, where the estimated policy gradient is
$\frac{1}{N} \sum_{n=1}^N \nabla_{\theta_{it}} \log \pi_{\theta_{it}} (a_{n} | s_{n}) {A}_{\pi_{it}} (s_{n}, a_{n}),$
given a batch of samples $\left\{(s_n,a_n,r_n,s_{n+1})\right\}$.
As discussed in the previous section, MPPO enables a better estimation of the advantage function with its mechanism utilizing the policy buffer.
Therefore, policy optimization drives each picked policy to explore previously unseen good states and actions.

\subsubsection {{Policy buffer update.}} Given the optimized policy at current iteration, the policy buffer needs to be updated.
A common way to update the policy buffer is to replace the worst policy in the policy buffer with the improved policy, as usually used in evolutionary-based methods for off-policy learning \cite{khadka2018evolution}. 
However, this updating scheme quickly loses the diversity of the policy buffer, and leads to a set of very similar policies ultimately, which results in a low exploration level, and will be further validated in Section \ref{sec:per_com}.
In MPPO, the updated policy will be added to the policy buffer by replacing the candidate policy itself, i.e., 
$\pi_{t+1}=(\pi_{1t},...,\pi_{it}^{'},...,\pi_{Kt}),$
which is able to maintain the diversity of the policy buffer.

The overall algorithm for the MPPO method is shown in Algorithm \ref{alg:mp_trpo}.
It is crucial to note that only the candidate policy interacts with the environment for sample collection, based on which both the candidate policy and the shared value network are updated. 

\subsection{Multi-Path Trust Region Policy Optimization}
We first apply our proposed MPPO method to a widely adopted on-policy algorithm TRPO \cite{schulman2015trust}, and obtain the resulting Multi-Path Trust Region Policy Optimization (MP-TRPO) algorithm.

Specifically, the update for the candidate policy is by backtracking line search with 
\begin{equation}
\theta_{i(t+1)} = \theta_{it} + \alpha^j \sqrt{\frac{2 \delta}{\hat{x}_{it}^T \hat{H}_{it} \hat{x}_{it}}} \hat{x}_{it},
\end{equation}
where $\hat{x}_{it} = \hat{H}_{it}^{-1} \hat{g}_{it}$ is computed by the conjugate gradient algorithm, and $\hat{g}_{it}=\nabla_{\theta} \mathcal{L}_{{\pi}_{\theta_{it}}}(\pi_{\theta}) |_{\theta=\theta_{it}}$ is the estimated policy gradient.
Note that the performance of the updated policy $J_i(\pi_{t+1})$ is estimated by $J_i(\pi_t) + \mathcal{G}_{\pi_{\theta_{it}}}(\pi_{\theta_{i(t+1)}})$.
Therefore, MP-TRPO does not require extra samples to evaluate the updated policy.

During the course of policy optimization, if the same policy is picked as in last iteration, MP-TRPO guarantees a monotonic improvement of the policy picked in current iteration over that in last iteration by \cite{schulman2015trust}.
On the other hand, if a policy that is more explorative but the performance is not as good as that in last iteration is picked, it may lead to a temporary performance drop.
In Theorem \ref{thm:MP_TRPO}, we show that such a performance drop can be bounded, ensuring a smooth policy transition.

\begin{theorem}
Let $i$, $j$ denote the indexes of policies that are picked at timestep $t$, $t+1$, respectively. 
Denote the improvement of $J_i(\pi_{t+1})$ over $J_i(\pi_t)$ as $\sigma_t$.
Then, the following bound holds for $0\leq \alpha<1$:
$J_j(\pi_{t+1}) - J_i(\pi_t) \geq \frac{-\alpha}{1-\alpha} \left[ \max_{k} J_k(\pi_{t+1}) - \min_{k} J_k(\pi_{t+1}) \right] + \sigma_t.$
\label{thm:MP_TRPO}
\end{theorem}

\begin{proof}
As $\pi_j$ is the policy selected at timestep $t+1$, we have
\begin{equation}
f_j(\pi_{t+1}) > f_i(\pi_{t+1}).
\end{equation}
Thus, 
\begin{equation}
\begin{split}
\hat{J}_j(\pi_{t+1}) - \hat{J}_i(\pi_{t+1}) > \frac{-\alpha}{1 - \alpha} \left( \hat{\mathcal{H}}_j(\pi_{t+1})  - \hat{\mathcal{H}}_i(\pi_{t+1}) \right) \geq \frac{-\alpha}{1-\alpha}.
\end{split}
\end{equation}
According to the min-max normalization, we have
\begin{equation}
\hat{J}_j(\pi_{t+1}) = \frac{J_j(\pi_{t+1}) - \min_k J_k(\pi_{t+1})}{\max_k J_k(\pi_{t+1}) - \min_k J_k(\pi_{t+1})}.
\end{equation}
Then, we obtain
\begin{equation}
\begin{split}
J_j(\pi_{t+1}) - J_i(\pi_{t+1})
\geq \frac{-\alpha}{1-\alpha} \left[ \max_k J_k(\pi_{t+1}) - \min_k J_k(\pi_{t+1}) \right].
\end{split}
\end{equation}
According to the monotonic improvement theorem \cite{schulman2015trust}, we have
\begin{equation}
\begin{split}
J_j(\pi_{t+1}) - J_i(\pi_t)
\geq \frac{-\alpha}{1-\alpha} \left[ \max_k J_k(\pi_{t+1}) - \min_k J_k(\pi_{t+1}) \right] + \sigma_t.
\end{split}
\end{equation}
\end{proof}

Theorem \ref{thm:MP_TRPO} shows that although there may be a temporary performance drop due to switching to a more explorative policy, such a sacrifice is bounded by an $\alpha$-related term and the difference of the performance of the best and the worst policies in current policy buffer.

\begin{figure}[!h] 
\centering
\subfloat[Visualization of the course of policy picking.]{\includegraphics[width=0.8\linewidth]{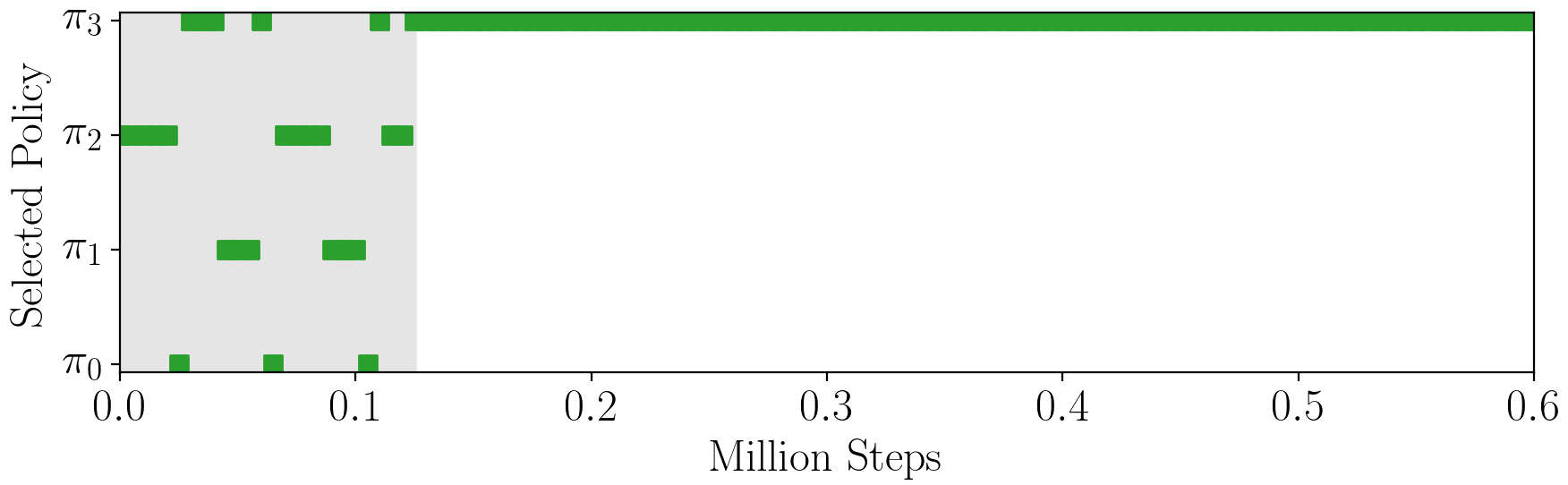}}
\\
\subfloat[Return of each policy.]{\includegraphics[width=0.45\linewidth]{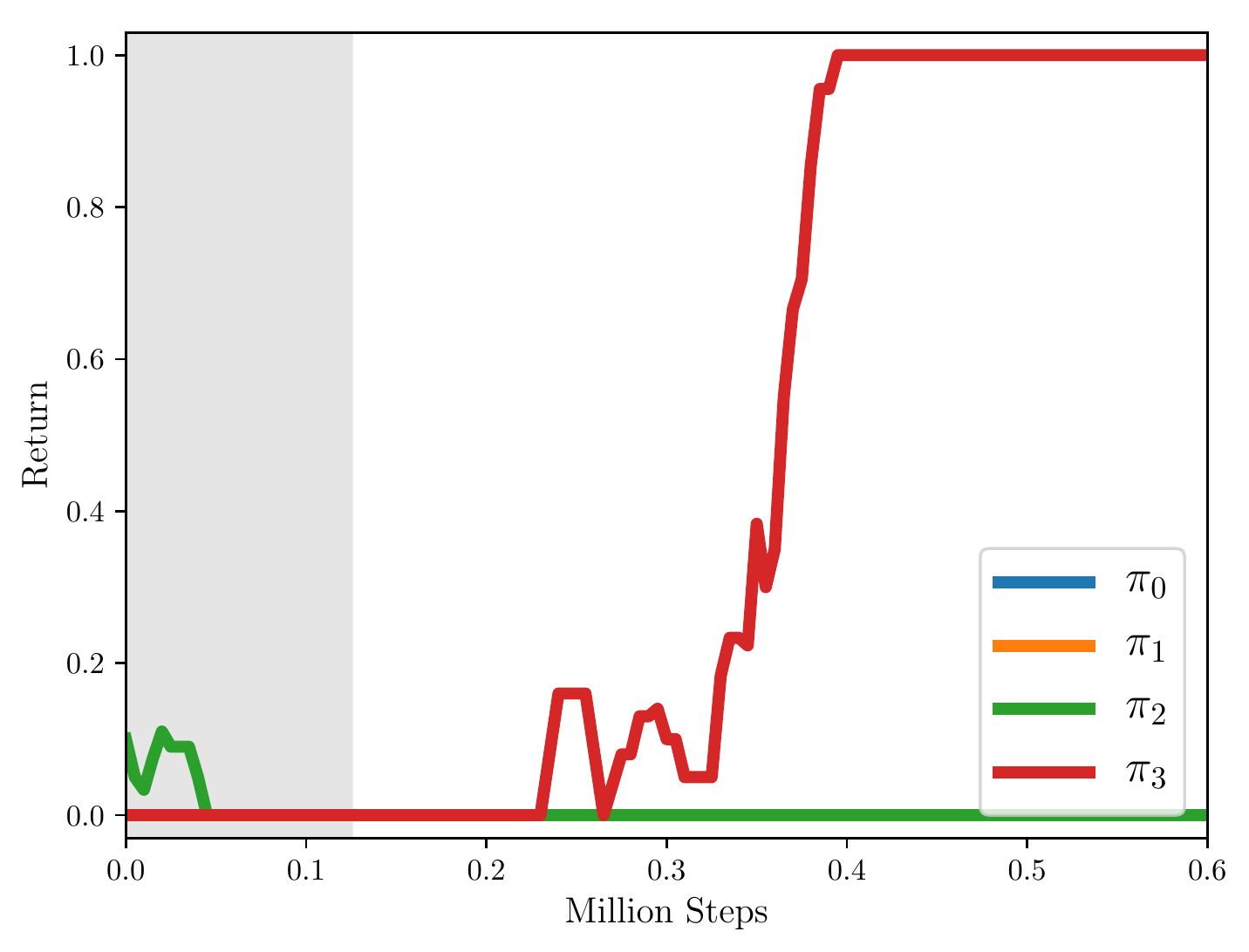}}
\subfloat[Entropy of each policy.]{\includegraphics[width=0.45\linewidth]{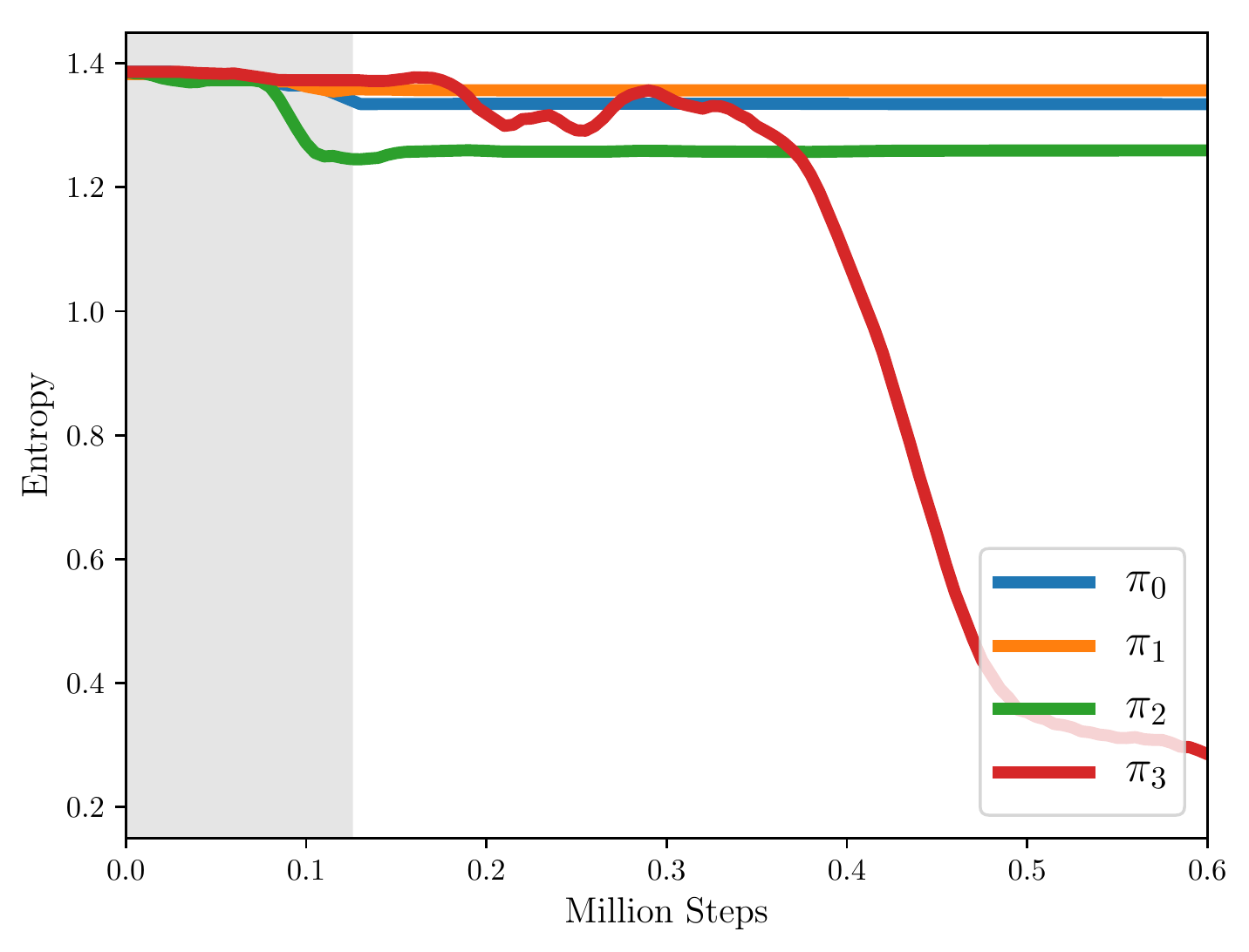}}
\caption{Learning details of MP-TRPO on Maze.}
\label{fig:selected_policy}
\end{figure}

Figure \ref{fig:selected_policy}(a) shows the learning process of MP-TRPO on Maze (Figure \ref{fig:maze}(a)) for a single seed.
In the beginning, MP-TRPO may pick different policies to collect samples and to optimize according to the picking rule, which provides diverse samples for updating the shared value network to better estimate the value function. 
This phase corresponds to the shaded region in Figure \ref{fig:selected_policy}(a), which leads to the fact that the performance gap between the best and the worst policies in the policy buffer is small (Figure \ref{fig:selected_policy}(b)).
Note that we use a fixed value of $\alpha$ to be $0.1$ in our experiments.
Therefore, the temporary performance drop is very small by Theorem \ref{thm:MP_TRPO}.
Note that a better estimation of the value function drives the optimization of the picked policy. 
In the end, MP-TRPO will converge to picking a single policy, in which case the performance of the picked policy will be monotone increasing. 
This observation actually holds for other seeds, and the full empirical result is referred to Appendix A.

We remark here that our method maintains good performance throughout the policy optimization process, while bringing the advantage of better exploration.

\subsection{Multi-Path Proximal Policy Optimization}
We also apply MPPO to another on-policy algorithm PPO \cite{schulman2017proximal}, and obtain the MP-PPO algorithm. 
To be specific, the candidate policy`s parameter is updated by stochastic gradient descent according to 
\begin{equation}
\theta_{i(t+1)}= \theta_{it} +  \eta \hat{g}_{it}, 
\end{equation}
where $\eta$ is the learning rate, and $\hat{g}_{it}=\nabla_{\theta} \mathcal{L}_{{\pi}_{\theta_{it}}}(\pi_{\theta}) |_{\theta=\theta_{it}}$ is the policy gradient estimated according to Eq. (\ref{eq:ppo}).

\section{Experiments}
We conduct extensive experiments to investigate the following key questions:
\begin{itemize}
\item How does MPPO compare with single-path policy optimization and state-of-the-art exploration methods? 

\item What is the effect of the number of paths $K$ and the weight $\alpha$?

\item Which component of MPPO is critical for the improvement of the exploration ability?

\item Is MPPO generally applicable given a baseline on-policy reinforcement learning algorithm to encourage exploration?
\end{itemize}

\subsection{Experimental Setup}
We evaluate MPPO on several continuous control environments simulated by the MuJoCo framework \cite{todorov2012mujoco}, which is a standard and widely-used benchmark for evaluating deep reinforcement learning algorithms \cite{duan2016benchmarking}.
MuJoCo tasks exhibit dense rewards, where the agent receives a reward at each step.
To better examine the exploration ability of our method, we further conduct evaluation in some more challenging variants of the original environments with sparse rewards \cite{houthooft2016vime,fu2017ex2,kang2018policy,plappert2018parameter,gangwani2019learning}.
For example, in {\sc SparseDoublePendulum}, a reward of $+1$ is given only when the agent reaches the goal that it swings the double pendulum upright, and $0$ otherwise.
Detailed descriptions of the benchmark environments are referred to Appendix B.
Each algorithm is run with $6$ different random seeds $(0$-$5)$, and the performance is evaluated for $10$ episodes every $10,000$ steps.
Note that the performance of MPPO is evaluated by the picked policy.
The averaged return in evaluation is reported as the solid line, with the shaded region denoting a $75\%$ confidence interval.
For fair comparisons, the hyper-parameters for all comparing algorithms are set to be the same as the best set of hyper-parameters reported in \cite{henderson2018deep}.
Please refer to Appendix B \cite{appendix} for implementation details.

\subsection{Baselines}
To comprehensively study the MP-TRPO algorithm, we compare it with six baselines.
For fair comparison, all methods use the same amount of $N$ samples during the course of policy optimization. 
\begin{itemize}
\item {\bf TRPO \cite{schulman2015trust}.} Vanilla single-path TRPO algorithm.
\item {\bf Curiosity-TRPO \cite{pathak2017curiosity}.} The curiosity-driven approach, which is a state-of-the-art method for exploration by augmenting the reward function with learned intrinsic rewards.
\item {\bf Diversity (Div)-TRPO \cite{hong2018diversity}.} The diversity-driven approach, which is also a state-of-the-art exploration method that augments the loss function of the policy with the distance of current policies and prior policies.
\item {\bf Multi-TRPO.} A baseline method that trains multiple ($K$) single-path TRPO with a shared value network and chooses the best one. We compare with the method to isolate the effect of the policy ensemble.
\item {\bf Multi-TRPO (Independent).} A baseline method training $K$ single-path TRPO, where each policy has its own value network, resulting in $K$ independent value networks in total. We compare with the method is to validate the effect of the shared value network.
\item {\bf MP-TRPO (ReplaceWorst).} A variant of MP-TRPO which updates the policy buffer by replacing the worst policy with the improved candidate policy. We compare with the method to evaluate the importance of the replacement strategy for updating the policy buffer.
\end{itemize}
We also verify the effectiveness of the picking rule by using different weights of $\alpha$ in MP-TRPO.

Then, we apply our proposed multi-path policy optimization mechanism to another baseline policy optimization method, PPO, to demonstrate the general applicability of the MPPO method, and conduct similar evaluation.

\subsection{Ablation Study}

\subsubsection {The effect of the number of paths $K$.} Figure \ref{fig:ablation_k} shows the performance of MP-TRPO and MP-PPO with varying $K$ on {\sc SparseDoublePendulum}.
The $K$ value trades off the diversity of the policy buffer and sample efficiency. 
A larger $K$ maintains a greater diversity, but may require more samples to learn as there are more policies to be picked and to be optimized in early periods of learning.
Indeed, there is an intermediate value for $K$ that provides the best trade-off.
We find that MP-TRPO with $K=8$ achieves the best performance and thus we fix $K$ to be $8$ on all environments.
For MP-PPO, a relatively smaller $K=2$ is sufficient and performs best.
This is because PPO itself exhibits greater exploration ability than TRPO, so we choose $K$ to be $2$ in all environments for MP-PPO. 
Note that MPPO with different values of $K$ all outperform the corresponding baseline policy optimization method (TRPO or PPO).

It is also worth noting that MPPO does not incur much more memory consumption for the population of policies in the policy buffer, where it only uses 1.67\% and 4.78\% more memory for MP-TRPO ($K=8$) and MP-PPO ($K=2$) compared with TRPO and PPO respectively on {\sc SparseDoublePendulum}. 
The summary of memory consumption for different $K$ is referred to Appendix C.

\begin{figure}[!h] 
\centering
\subfloat[Varying $K$ for MP-TRPO.]{\includegraphics[width=0.5\linewidth]{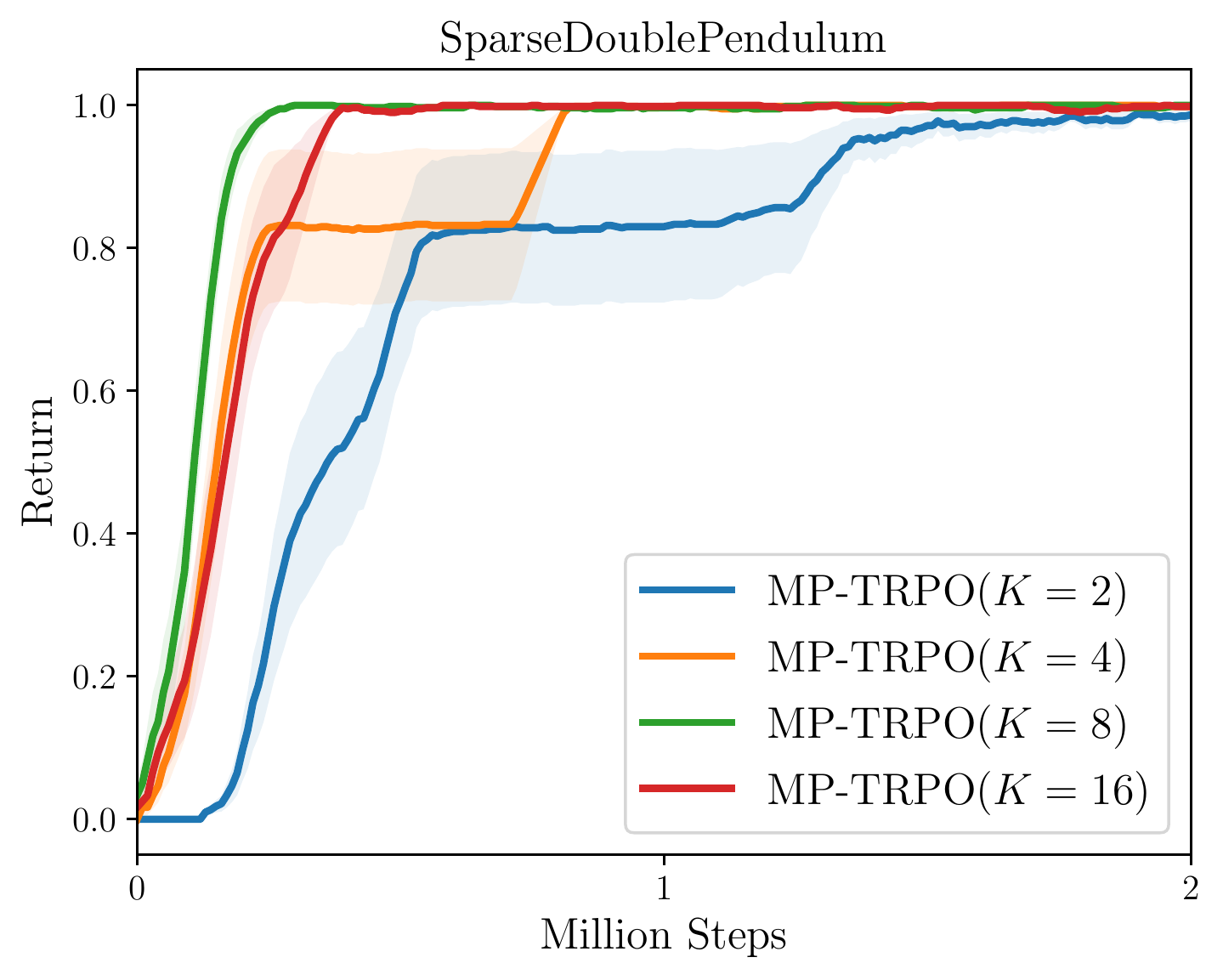}}
\subfloat[Varying $K$ for MP-PPO.]{\includegraphics[width=0.5\linewidth]{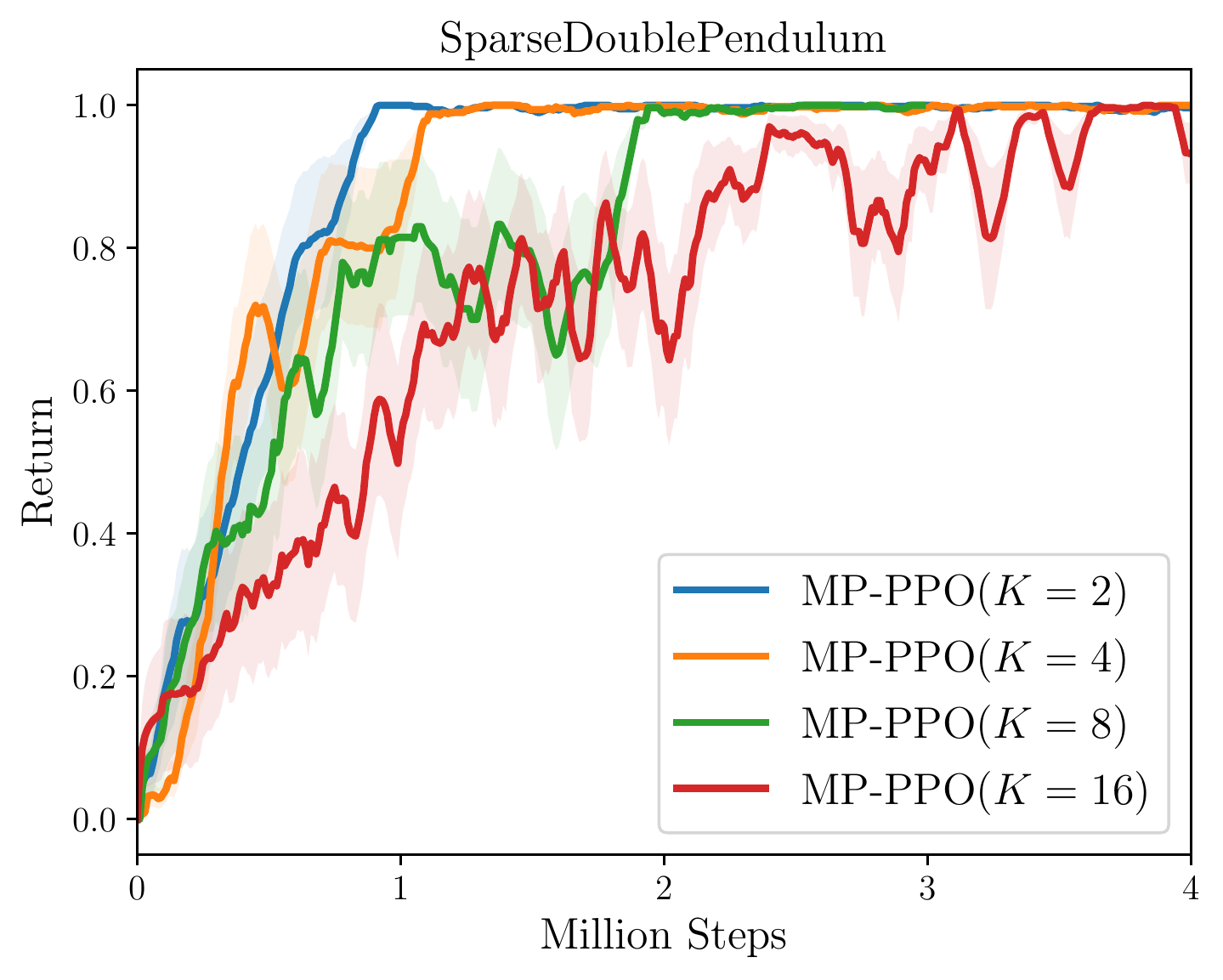}}
\caption{Ablation study of varying $K$.}
\label{fig:ablation_k}
\end{figure}

\subsubsection {The effect of the weight $\alpha$.} In the picking rule, $\alpha$ controls the trade-off between exploration and exploitation.
A larger $\alpha$ emphasizes more on the exploration ability of the picked policy, but may fail to utilize the result of policy optimization.
In addition, according to Theorem \ref{thm:MP_TRPO}, a large $\alpha$ may lead to a temporary performance drop.
On the other hand, a smaller $\alpha$ focuses more on exploiting the current best-performing policy in the policy buffer, where $\alpha=0$ refers to always picking the best policy based on current estimation.
We vary $\alpha$ for MP-TRPO on {\sc SparseDoublePendulum}, and the result is shown in Figure \ref{fig:ablation_alpha}.
As expected, a small $\alpha=0.1$ achieves the best performance, so we fix $\alpha$ to be $0.1$ in all environments for both MP-TRPO and MP-PPO.

\begin{figure}[!h]
\centering
\includegraphics[scale=0.28]{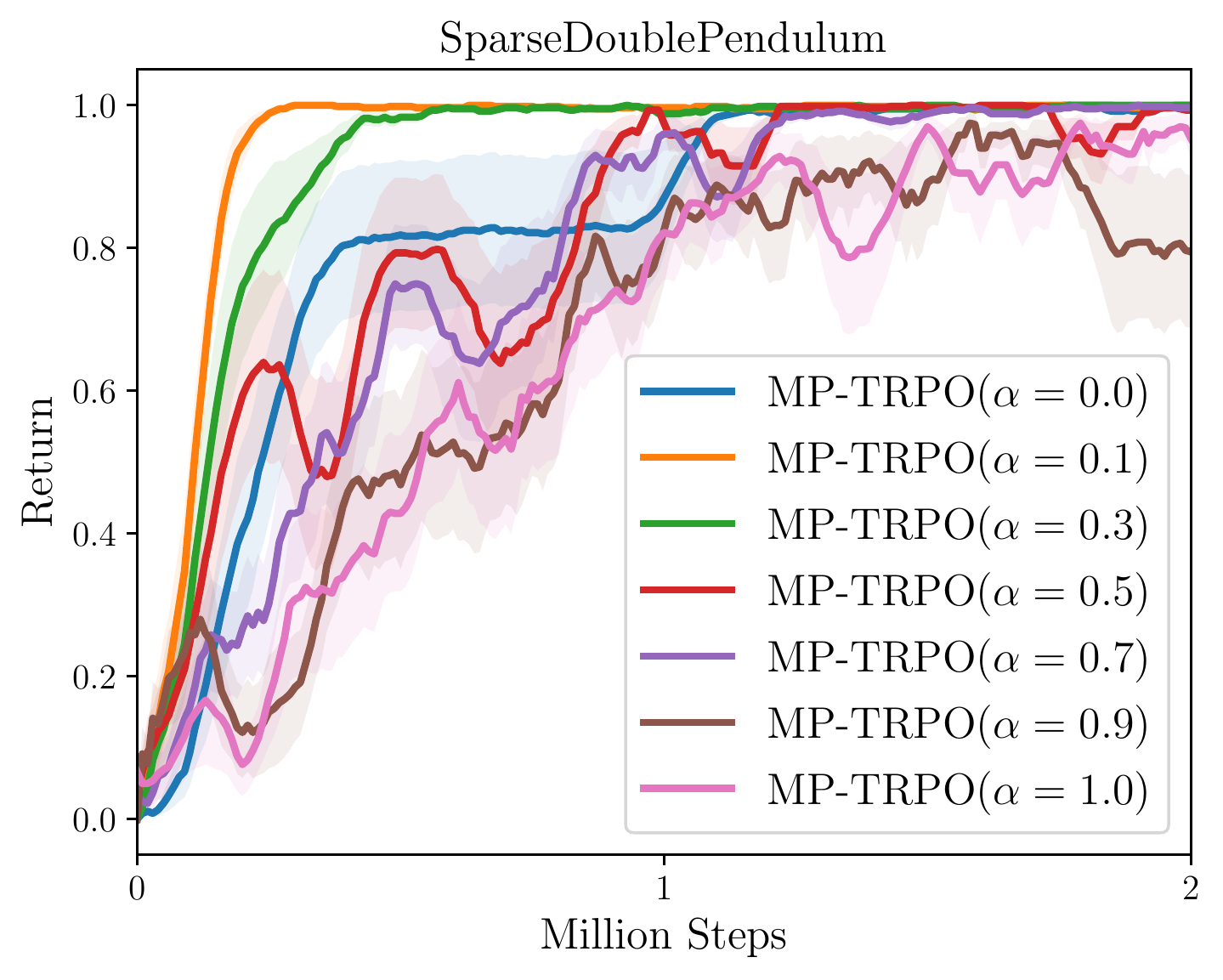}
\caption{Ablation study of varying $\alpha$.}
\label{fig:ablation_alpha}
\end{figure}

\subsection{Performance Comparison} \label{sec:per_com}
\begin{figure*}
  \centering
  \includegraphics[scale=0.25]{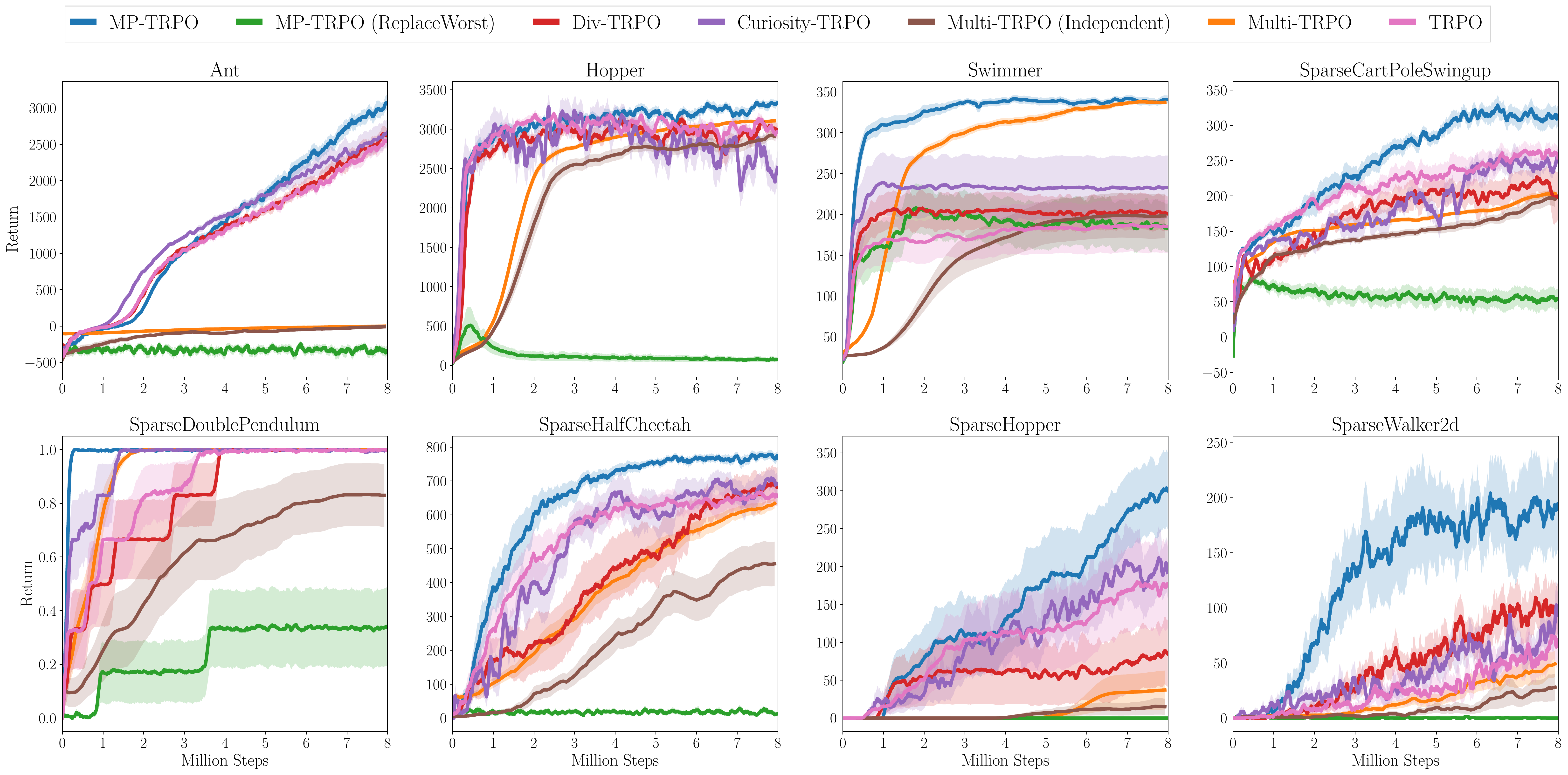}
    \caption{Performance comparison of MP-TRPO.}    
    \label{fig:trpo}
\end{figure*}

\begin{table*}[!h]
  \footnotesize
  \centering
  \caption{Comparison of MP-TRPO on final performance (mean and confidence interval). }
  \begin{tabular}{lccccccc}
    \toprule
    \multirow{2}*{Environment} & \multirow{2}*{MP-TRPO} & MP-TRPO & \multirow{2}*{Div-TRPO} & \multirow{2} * {Curiosity-TRPO} & {Multi-TRPO} & \multirow{2} *{Multi-TRPO} & \multirow{2}*{TRPO} \\
    ~ & ~ & (ReplaceWorst) & ~ & ~ & (Independent) & ~ & ~ \\
    \midrule
    \sc Ant & {\bf 3017.02} (116.405) & -470.05 (57.65) & 2635.00 (138.28) & 2744.79 (190.23) & 0.93 (1.99) & 3.90 (1.19) & 2558.33 (120.89) \\
    \sc Hopper & {\bf 3257.55} (62.68) & 76.25 (32.68) & 2957.31 (82.94) & 2786.99 (258.17) & 2878.34 (91.57) & 3138.41 (17.66) & 2947.19 (111.43)\\
    \sc Swimmer & {\bf 340.92} (4.39) & 179.67 (29.63) & 199.06 (20.58) & 232.89 (38.93) & 198.00 (27.47) & 339.17 (2.21) & 186.21 (32.63) \\
    \sc SparseCartPoleSwingup & {\bf 320.47} (14.48) & 49.10 (16.99) & 203.22 (33.56) & 244.23 (11.42) & 180.17 (11.50) & 213.07 (4.07) & 238.88 (11.08)\\
    \sc SparseDoublePendulum & {\bf 1.00} (0.00) & 0.33 (0.15) & {\bf 1.00} (0.00) & {\bf 1.00} (0.00) & 0.83 (0.12) & 0.98 (0.01) & {\bf 1.00} (0.00) \\
    \sc SparseHalfCheetah & {\bf 756.02} (14.02) & 24.52 (8.14) & 699.03 (57.15) & 656.98 (29.24) & 438.50 (66.98) & 676.22 (14.35) & 639.32 (38.40) \\
    \sc SparseHopper & {\bf 302.32} (52.48) & 0.00 (0.00) & 89.60 (45.09) & 188.47 (33.51) & 14.63 (10.29) & 38.33 (25.82) & 182.63 (60.13) \\
    \sc SparseWalker2d & {\bf{186.48}} (43.15) & 0.00 (0.00) & 112.75 (33.29) & 87.33 (25.18) & 37.67 (12.57) & 56.18 (17.27) & 53.78 (15.57) \\
    \bottomrule
  \end{tabular}
  \label{tab:table_comparison}
\end{table*}

{\bf Comparative analysis.}
The comparative results of MP-TRPO are demonstrated in Figure \ref{fig:trpo}.
As shown, MP-TRPO is consistently more sample efficient than Div-TRPO in all environments.
In addition, it outperforms Curiosity-TRPO in all but one environment in terms of sample efficiency.
The margin is larger especially in sparse environments.
Table 1 summarizes the performance at the end of training, which shows that MP-TRPO achieves the best final performance in all environments.

Div-TRPO augments the loss function with a measure of distances between past policies.
As trust-region methods limit the update, the distance among past policies is not large.
Thus, the diversity-driven technique does not enable significant improvement of exploration on TRPO.
Curiosity-TRPO augments the reward function with a curiosity term that measures how novel a state is. 
It encourages the agent to re-explore states that are known to be unfamiliar with. 
However, it can be challenging for the agent to first discover such states in sparse environments. \\

\begin{figure*}[!h]
  \centering
  \includegraphics[scale=0.24]{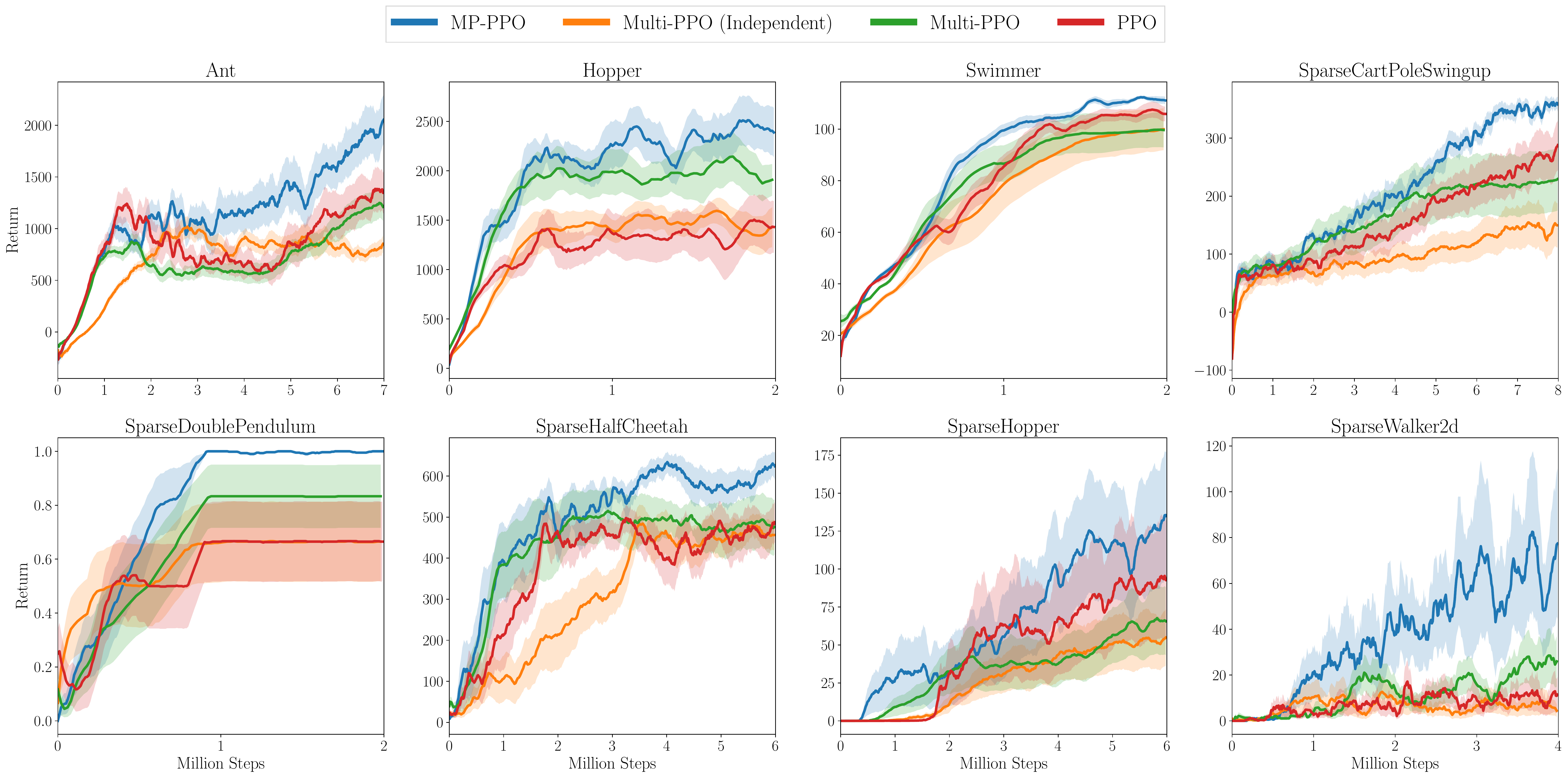}
    \caption{Performance comparison of MP-PPO.}    
    \label{fig:ppo}
\end{figure*}

\begin{table*}[!h]
    \footnotesize
  \centering
  \caption{Comparison of MP-PPO on eventual performance (mean and confidence interval).}
  \begin{tabular}{lcccc}
    \toprule
    Environment & MP-PPO & Multi-PPO (Independent) & Multi-PPO & PPO \\
    \midrule
    \sc Ant & {\bf 1992.21} (216.33) & 902.31 (101.00) & 1131.90 (110.94) & 1311.35 (220.93) \\
    \sc Hopper & {\bf 2264.54} (234.95) & 1449.33 (224.93) & 2073.46 (161.47) & 1457.18 (258.87) \\
    \sc Swimmer & {\bf 109.96} (1.93) & 100.28 (9.08) & 98.79 (6.99) & 106.12 (2.78) \\
    \sc SparseCartPoleSwingup & {\bf 352.12} (16.14) & 135.65 (37.47) & 238.02 (55.98) & 294.88 (43.42) \\
    \sc SparseDoublePendulum & {\bf 1.0} & 0.67 (0.15) & 0.83 (0.12) & 0.67 (0.15) \\
    \sc SparseHalfCheetah & {\bf 593.37} (40.60) & 456.65 (35.86) & 478.70 (59.02) & 463.07 (36.72) \\
    \sc SparseHopper & {\bf 131.35} (40.08) & 57.88 (19.86) & 63.75 (22.69) & 90.40 (37.61) \\
    \sc SparseWalker2d & {\bf 55.93} (31.98) & 1.67 (1.17) & 41.68 (19.91) & 10.20 (7.18) \\
    \bottomrule 
  \end{tabular}
  \label{tab:table_comparison_ppo}
\end{table*}

\noindent {\bf Effect of each component.}
Regarding the shared value network, Multi-TRPO outperforms Multi-TRPO (Independent), as it enables a better estimation of the value function.
Additionally, MP-TRPO outperforms Multi-TRPO significantly, which validates that MP-TRPO enables efficient exploration with its mechanism to utilize the population of policies.
As for the strategy for policy buffer updates, note that MP-TRPO ({ReplaceWorst}), where replacing the worst policy is a common strategy in evolutionary-based methods \cite{khadka2018evolution}, performs poorly in all but one benchmark environments.  
After updating the picked policy, it replaces the worst-performing policy in the buffer with this improved policy.
Under this updating scheme, the policy buffer loses the diversity of policies quickly and soon only stores $K$ similar copies of a single policy.
Thus, MP-TRPO ({ReplaceWorst}) performs worse than {Multi-TRPO (Independent)}.
In contrast, the replacement strategy of MP-TRPO best preserves the diversity of the policy buffer while ensuring policy optimization.

Our results provide empirical evidence that MPPO is an efficient mechanism to fully utilize the population of policies without bringing high computation overhead. \\

\noindent {\bf Performance based on PPO.}
To show that MPPO is readily applicable to other baseline on-policy algorithms,
we build it upon PPO, and evaluate the resulting MP-PPO algorithm by comparing it with the corresponding {PPO}, {Multi-PPO}, and Multi-PPO (Independent) algorithms.
Learning curves are shown in Figure \ref{fig:ppo}, with the final performance summarized in Table \ref{tab:table_comparison_ppo}. 
Results show that MP-PPO outperforms the baseline methods in all environments in terms of sample efficiency and final performance, which demonstrates its effectiveness to encourage exploration.

\section{Related Work}
In reinforcement learning, entropy is a critical term which relates to the uncertainty of the policy.
Entropy-regularized reinforcement learning \cite{haarnoja2017reinforcement,haarnoja2018soft,nachum2017trust} optimizes the standard objective augmented by an entropy regularizer considering the distance with the random policy, and thus learns a stochastic policy for better exploration.  
Our method differs from them in that MPPO still optimizes the standard objective, where the picking rule involves the performance and an entropy bonus term.

Evolutionary methods have emerged to be an alternative of deep reinforcement learning \cite{salimans2017evolution,such2017deep}, and recent works that combine evolutionary methods and reinforcement learning \cite{khadka2018evolution,gangwani2017policy} have shown great power for better exploration and stability.
There have also been a number of approaches improving exploration by combining evolutionary methods with deep reinforcement learning by maintaining a population of agents.
Gangwani et al. \cite{gangwani2017policy} apply policy gradient methods to mutate the population.
Khadka et al. \cite{khadka2018evolution} utilize a population of evolutionary actors to collect samples, where a reinforcement learning actor based on DDPG \cite{lillicrap2015continuous} is updated using these samples.
Pourchot et al. \cite{pourchot2018cem} propose to combine the cross-entropy method and TD3 \cite{fujimoto2018addressing}.
Our work differs from previous works in several aspects. 
First, we train a single policy at each iteration instead of a population of policies, and only the picked policy interacts with the environment.
Second, we use multi-path to enable better exploration than single-path for on-policy algorithms, while previous works cannot be applied to on-policy algorithms.

Another approach related to our work is \cite{zhang2019ace}, where Zhang et al. propose to escape from local maxima for an off-policy algorithm, DDPG \cite{lillicrap2015continuous}, by utilizing an ensemble of actors.
The critic is updated according to the best action proposed by all actors that results in maximum Q-value, and all actors are trained in parallel.
However, it cannot be applied to RL algorithms with stochastic policies.

\section{Conclusion}
We present Multi-Path Policy Optimization (MPPO), which uses a population of policies to improve exploration for on-policy reinforcement learning algorithms.
We apply the MPPO method to TRPO and PPO, and show that the performance can be guaranteed during policy switching.
We conduct extensive experiments on several MuJoCo tasks including environments with sparse rewards, and show that MPPO outperforms baselines significantly in both sample efficiency and final performance.

\begin{acks}
The work of  Ling Pan and Longbo Huang was supported in part by the National Natural Science Foundation of China Grant 61672316, the Zhongguancun Haihua Institute for Frontier Information Technology and the Turing AI Institute of Nanjing.
\end{acks}

\appendix

\section{Visualization of the Picking Rule}
The picked policies chosen by the picking rule of MP-TRPO on Maze during the first 0.5 million steps (the total number of training steps is 1 million) by different random seeds (0-5) is shown in Figure \ref{fig:vis}.
The x-axis and y-axis correspond to the training steps and the index of the picked policies.
As shown, in the beginning of learning, different policies are picked according to the picking rule, which is a weighted objective of performance and entropy.
Finally, MPPO converges to picking a same policy to optimize.
\begin{figure}[!h]
    \centering
    \includegraphics[scale=0.34]{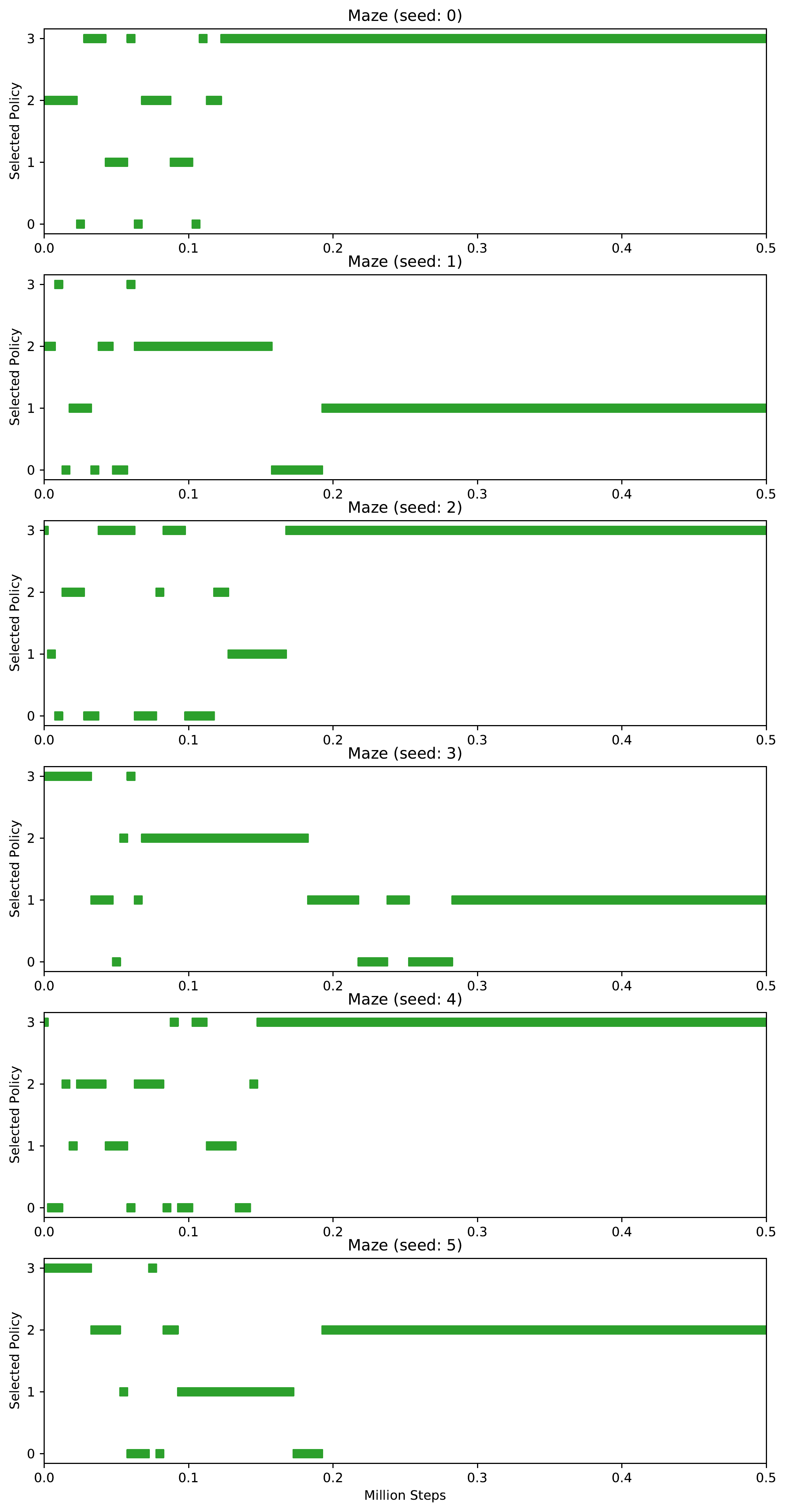}
    \caption{Visualization of the picked policies of MP-TRPO on Maze during the first 0.5M steps.}    
    \label{fig:vis}
\end{figure}

\section{Details of Experimental Setup}
\subsection{Environments}
The environments are all from OpenAI Gym \cite{brockman2016openai}, and the details of the sparse environments are summarized as follows: 
\begin{itemize}
\item {\sc SparseCartPoleSwingup}: a reward of $+1$ is given only when $\cos(\beta) > 0.8$, where $\beta$ is the pole angle, and $0$ otherwise
\item {\sc SparseDoublePendulum}: a reward of $+1$ is given only when the agent reaches the goal, i.e. swings the double pendulum upright, and $0$ otherwise 
\item {\sc SparseHalfCheetah}: the agent only receives a reward only when it runs multiple meters above the thereshold, and $0$ otherwise
\item {\sc SparseHopper}: the agent only receives a reward only when it hops multiple meters above the thereshold, and $0$ otherwise
\item {\sc SparseWalker2d}:  the agent only receives a reward only when it walks multiple meters above the thereshold, and $0$ otherwise
\end{itemize}

\subsection{Hyperparamters}
The hyper-parameters for MP-TRPO and TRPO, MP-PPO and PPO are shwon in Table 1 and Table 2 respectively, which are set to be the same for fair comparison according to \cite{henderson2018deep}.
For all algorithms, the policy network is (64, tanh, 64, tanh, linear), and the value network is (64, tanh, 64, tanh, linear).
The size of the policy buffer $K$ is set to be $8$ and $2$ in MP-TRPO and MP-PPO respectively, and the weight parameter $\alpha$ is set to be 0.1 in all environments.

\begin{table}[!h]
	\centering
	\caption{Hyper-parameters of MP-TRPO and TRPO.}
	\begin{tabular}{lc}
		\toprule
		{ Hyper-parameter} & { Value} \\
		\midrule
		Discount Factor $\gamma$ & 0.995 \\
		GAE $\lambda$ & 0.97 \\
		Batch Size & 5000 \\
		Iterations of Conjugate Gradient & 20 \\
		Damping of Conjugate Gradient & 0.1 \\
		Iterations of Value Function Update & 5 \\
		Batch Size of Value Function Update & 64 \\
		Step Size of Value Function Update & 0.001 \\
		Coefficient of Entropy & 0.0 \\
		max KL & 0.01 \\
		\bottomrule
	\end{tabular}
	\label{tab:hyper_trpo}
\end{table}

\begin{table}[!h]
	\centering
	\caption{Hyper-parameters of MP-PPO and PPO.}
	\begin{tabular}{lc}
		\toprule
		{ Hyper-parameter} & { Value} \\
		\midrule
		Discount Factor $\gamma$ & 0.995 \\
		GAE $\lambda$ & 0.97 \\
		Batch Size & 2048 \\
		Clip Parameter $\epsilon$ & 0.2 \\
		Epochs of Optimizer per Iteration & 10 \\
		Step Size of Optimizer & 0.0003 \\
		Batch Size of Optimizer & 64 \\
		Coefficient of Entropy & 0.0 \\
		\bottomrule
	\end{tabular}
	\label{tab:hyper_ppo}
\end{table}

\section{Memory Consumption}
The population of policies does not incur much more memory consumption, and comparison results for MP-TRPO and MP-PPO with varing number of paths $K$ on SparseDoublePendulum corresponding to our ablation experiments are shown in Table 3 and Table 4 respectively.

\begin{table}[!h]
\centering
\label{tab:mem_mptrpo}
\caption{Comparison results of memory consumption for MP-TRPO with different $K$.}\
\begin{tabular}{cccc}
\toprule
& GPU Memory & Memory \\
\midrule
TRPO & 359 M & 404 M \\
MP-TRPO ($K=2$) & 365 M & 410 M \\
MP-TRPO ($K=4$) & 365 M & 410 M\\
MP-TRPO ($K=8$) & 365 M & 410 M\\
MP-TRPO ($K=16$) & 365 M & 411 M\\
\bottomrule
\end{tabular}
\end{table}

\begin{table}[!h]
\centering
\label{tab:mem_mpppo}
\caption{Comparison results of memory consumption for MP-PPO with different $K$.}\
\begin{tabular}{cccc}
\toprule
& GPU Memory & Memory \\
\midrule
PPO & 335 M & 388 M \\
MP-PPO ($K=2$) & 351 M & 407 M \\
MP-PPO ($K=4$) & 351 M & 409 M\\
MP-PPO ($K=8$) & 351 M & 409 M\\
MP-PPO ($K=16$) & 335 M & 411 M\\
\bottomrule
\end{tabular}
\end{table}

\bibliographystyle{ACM-Reference-Format}  % do not change this line!
\bibliography{mppo}  % put name of your .bib file here

\end{document}